\documentclass[letterpaper]{article} 
\usepackage{aaai24}  
\usepackage{times}  
\usepackage{helvet}  
\usepackage{courier}  
\usepackage[hyphens]{url}  
\usepackage{graphicx} 
\usepackage{natbib}  
\usepackage{caption} 
\usepackage{algorithm}
\usepackage{algorithmic}
\usepackage{multirow} 
\usepackage{amsmath,amssymb,amsfonts,amsthm}  
\newtheorem{theorem}{Theorem}
\newtheorem{lemma}{Lemma}
\newtheorem{proposition}[theorem]{Proposition}

\newtheorem{definition}{Definition}
\newtheorem{innercustomgeneric}{\customgenericname}
\providecommand{\customgenericname}{}
\newcommand{\newcustomtheorem}[2]{%
  \newenvironment{#1}[1]
  {%
   \renewcommand\customgenericname{#2}%
   \renewcommand\theinnercustomgeneric{##1}%
   \innercustomgeneric
  }
  {\endinnercustomgeneric}
}

\newcustomtheorem{customthm}{Theorem}
\newcustomtheorem{customlemma}{Lemma}
\newcustomtheorem{customcorollary}{Corollary}
\newcustomtheorem{customdef}{Definition}
\usepackage{comment}
\usepackage{newfloat}
\usepackage{listings}

\usepackage{caption}
\DeclareCaptionStyle{ruled}{labelfont=normalfont,labelsep=colon,strut=off} 
\floatstyle{ruled}
\newfloat{listing}{tb}{lst}{}
\floatname{listing}{Listing}
\title{Sparse Variational  Student-t Processes}
\author{
    Jian Xu,
    Delu Zeng\thanks{ Corresponding author.}}
\affiliations{
    South China University of Technology\\

    381 Wushan Road, Tianhe District \\
   Guangzhou City, Guangdong Province, China\\
    2713091379@qq.com, dlzeng@scut.edu.cn
}

\usepackage{bibentry}

\begin{document}

\maketitle

\begin{abstract}
The theory of Bayesian learning incorporates the use of Student-t Processes to model heavy-tailed distributions and datasets with outliers. However, despite Student-t Processes having a similar computational complexity as Gaussian Processes, there has been limited emphasis on the sparse representation of this model. This is mainly due to the increased difficulty in modeling and computation compared to previous sparse Gaussian Processes. Our motivation is to address the need for a sparse representation framework that reduces computational complexity, allowing Student-t Processes to be more flexible for real-world datasets. To achieve this, we leverage the conditional distribution of Student-t Processes to introduce sparse inducing points. Bayesian methods and variational inference are then utilized to derive a well-defined lower bound, facilitating more efficient optimization of our model through stochastic gradient descent. We propose two methods for computing the variational lower bound, one utilizing Monte Carlo sampling and the other employing Jensen's inequality to compute the KL regularization term in the loss function. We propose adopting these approaches as viable alternatives to Gaussian processes when the data might contain outliers or exhibit heavy-tailed behavior, and we provide specific recommendations for their applicability. We evaluate the two proposed approaches on various synthetic and real-world datasets from UCI and Kaggle, demonstrating their effectiveness compared to baseline methods in terms of computational complexity and accuracy, as well as their robustness to outliers.
\end{abstract}

\section{Introduction}

 Gaussian Processes (GPs) \cite{rasmussen2003gaussian} offer a versatile approach for incorporating non-parametric priors into functions and have been extensively applied in various fields including time-series forecasting \cite{heinonen2018learning}, computer vision \cite{blomqvist2020deep}, and robotics \cite{deisenroth2013gaussian, lee2022trust}. However, a key limitation of GPs is their computational complexity, with an exact implementation scaling as  $\mathcal{O}(n^3
)$ time and $\mathcal{O}(n^2)$ memory, where $n$ is the number of training cases. Fortunately, recent advancements have focused on developing sparse approximations \cite{titsias2009variational} that maintain the desirable properties of Gaussian processes while significantly reducing computational costs. These sparse approximations achieve a computational complexity of $\mathcal{O}(nm^2)$ time and $\mathcal{O}(nm)$ memory, where $m$ is a value smaller than $n$ and represents the number of elements in the sparse approximations set.

Sparse approximations achieve this reduction by focusing inference on a small number of quantities, which provide an approximation of the entire posterior over functions. These quantities can be selected differently, such as function values at specific input locations \cite{quinonero2005unifying}, properties of spectral representations \cite{lazaro2010sparse}, or more abstract representations \cite{lazaro2009inter}. Similar approaches are also employed in random feature expansions \cite{cutajar2017random}. In this context, we specifically examine methods that approximate the posterior by using the function values at a set of $m$ inducing inputs \cite{moss2023inducing} (referred to as pseudo-inputs).

The success of sparse Gaussian processes has sparked interest in extending the methodology to encompass more general families of elliptical processes \cite{bankestad2020elliptical}, such as the Student-t process (TP) \cite{shah2014student}. The TP offers added flexibility and robustness against outliers, justified by  Bayesian learning theory \cite{tang2017student,chen2020multivariate,andrade2023robustness}. Although the computational complexity of TP is comparable to that of GPs, there has been limited research on developing sparse representation techniques for TP .

This can be attributed to several factors. Firstly, the probability density function of TP is more intricate compared to that of GPs, which presents challenges in deriving the conditional and marginal distributions for TP. Secondly, developing a sparse representation for TP is a novel problem that has not been adequately addressed in previous works. Moreover, there are unresolved issues related to posterior inference and algorithm design specific to TP. Consequently, the task of effectively modeling and computing TP using sparse techniques remains a significant challenge.

To address these challenges, this paper introduces a novel approach that extends the benefits of sparse representation from GPs to TP. Our proposed method, called Sparse Variational Student-t Processes (SVTP), leverages the conditional distribution of TP to incorporate sparse inducing points. By utilizing these inducing points, we construct the prior distribution and joint probability distribution of SVTP.

The key idea behind the Sparse Variational Student-t Process (SVTP) is to effectively summarize the information contained in the data using sparse inducing points. This allows for the application of variational inference techniques to approximate the posterior distribution, resulting in a reduction in computational complexity. This reduction in complexity enables the application of the Student-t Process to various real-world datasets, enhancing its practical utility. Additionally, SVTP is particularly significant in modeling processes that exhibit outliers and heavy-tailed behavior.

Specifically, in our approach, we utilize Bayesian methods and variational inference \cite{kingma2013auto, blei2017variational} to derive a well-defined lower bound for the SVTP model. This lower bound allows for more efficient optimization using stochastic gradient descent \cite{ketkar2017stochastic}, thereby enhancing the scalability of the model. Based on the non-analyticity of the KL divergence between two Student distributions \cite{roth2012multivariate}, we propose two methods for computing the variational lower bound: SVTP-UB and SVTP-MC. SVTP-UB utilizes Jensen's inequality to obtain an upper bound on the KL regularization term in the loss function, while SVTP-MC uses Monte Carlo sampling to evaluate it. We provide theoretical and experimental analysis to demonstrate the applicability of these methods and their respective advantages. Additionally, we leverage reparameterization tricks \cite{salimbeni2017doubly} to efficiently sample from the posterior distribution of SVTP, drawing inspiration from literature on Bayesian learning. Furthermore, we conduct a comparative analysis between SVTP and  SVGP and provide a theoretical explanation for why SVTP outperforms SVGP in handling outlier data. 

The two proposed approaches are evaluated on a series of synthetic and real-world datasets from UCI and Kaggle, demonstrating its effectiveness over baseline methods in terms of computational complexity, model accuracy, and robustness. The results suggest that the SVTP is a promising approach for extending the benefits of sparse representation to outliers and heavy-tailed distributions, offering a powerful tool for large-scale datasets in various applications.
Overall, our contributions are as follows:
\begin{itemize}
\item We propose a unified framework for sparse TP, which utilizes inducing points to obtain sparse representations of the data, aiming to reduce the complexity of TP.
\item By employing variational inference and stochastic optimization, we define a well-defined  ELBO and present two effective algorithms for inference and learning in our proposed model, namely SVTP-UB and SVTP-MC. We also analyze the theoretical connections and advantages of SVTP methods compared to SVGP.
\item We conduct experiments on eight real-world datasets, two synthetic datasets. These experiments include verification of time complexity, accuracy and uncertainty validation, regression on outlier datasets. The results across all experiments demonstrate the effectiveness and scalability of our proposed algorithm, particularly showcasing its robustness on outlier datasets.
\end{itemize}

\section{Background and Notations}
\subsection{Gaussian Processes and Sparse Representation}

Gaussian Processes (GPs) are widely used in regression  tasks due to their flexibility and interpretability. A GP is a collection of random variables, any finite number of which have a joint Gaussian distribution. For a regression problem, we observe a set of input-output pairs $\{\mathbf{x}_i, y_i\}_{i=1}^n$, where $y_i = f(\mathbf{x}_i) + \epsilon_i$ and $\epsilon_i$ is the noise term. We assume that $f(\mathbf{x})$ follows a GP with mean function $m(\mathbf{x})$ and covariance function $k(\mathbf{x}, \mathbf{x}')$. Then, the predictive distribution of the function values at a new test point $\mathbf{x}^*$ is given by:
\begin{equation}
\begin{aligned}
    &p(f(\mathbf{x}^*) | \mathbf{x}^*, \mathbf{X}, \mathbf{y}) \\= &\mathcal{N}(\mathbf{k}(\mathbf{x}^*, \mathbf{X})\mathbf{K}^{-1}\mathbf{y}, k(\mathbf{x}^*, \mathbf{x}^*) - \mathbf{k}(\mathbf{x}^*, \mathbf{X})\mathbf{K}^{-1}\mathbf{k}(\mathbf{X}, \mathbf{x}^*)),
\end{aligned}    
\end{equation}

where $\mathbf{k}(\mathbf{x}^*, \mathbf{X})$ is the $1 \times n$ covariance vector between the test point and the input points, and $\mathbf{K}$ is the $n \times n$ covariance matrix among the input points. The computation of the predictive distribution involves the inversion of the $n \times n$ covariance matrix $\mathbf{K}$, which has a cubic computational complexity of $\mathcal{O}(n^3)$ and a quadratic memory complexity of $\mathcal{O}(n^2)$. These complexity factors limit the usability of GPs for large-scale problems. 

To address this issue, sparse representation methods \cite{hensman2015scalable} have been proposed, which use a small number of inducing points $\mathbf{Z}$ to approximate the full covariance matrix of the GP. Sparse Gaussian processes can be formulated using various techniques. One popular approach is the Sparse Variational  Gaussian processes (SVGP), which use a variational distribution to approximate the true GP posterior. In SVGP, we introduce a set of $m$ inducing points $\mathbf{Z} = \{\mathbf{z}_i\}_{i=1}^m$ and corresponding function values $\mathbf{u} = \{u_i\}_{i=1}^m$ at these points. Then, we introduce a variational distribution $q(\mathbf{u})$ to approximate the true posterior distribution $p(\mathbf{u}|\mathbf{y})$ of the inducing function values. The choice of the variational distribution $q(\mathbf{u})$ is often tractable, such as a Gaussian distribution. Using the variational distribution $q(\mathbf{u})$, the predictive distribution at test points $\mathbf{x}^*$ is given by:
\begin{equation}
p(f(\mathbf{x}^*)| \mathbf{x}^*, \mathbf{X}, \mathbf{y}) = \int p(f(\mathbf{x}^*)|\mathbf{x}^*, \mathbf{Z},\mathbf{u})q(\mathbf{u})d\mathbf{u},
\end{equation}
This predictive distribution can be expressed using an efficient calculation of a low-dimensional matrix inverse, instead of the full GP covariance matrix:
\begin{equation}
\begin{aligned}
    &p(f(\mathbf{x}^*)|\mathbf{x}^*, \mathbf{Z},\mathbf{u}) \\=& \mathcal{N}( \mathbf{k}(\mathbf{x}^*,\mathbf{Z})\mathbf{K}_u^{-1}\mathbf{u}, k(\mathbf{x}^*,\mathbf{x}^*)-\mathbf{k}(\mathbf{x}^*,\mathbf{Z})\mathbf{K}_u^{-1}\mathbf{k}(\mathbf{Z},\mathbf{x}^*)),
\end{aligned}
\end{equation}
where $\mathbf{k}(\mathbf{x}^*,\mathbf{Z})$ is the $1\times m$ covariance vector between the test point and the inducing points, $\mathbf{K}_u$ is the $m\times m$ covariance matrix among the inducing points.
The inducing points can be selected to maximize the Evidence Lower BOund (ELBO) of the marginal likelihood $\log p(\mathbf{y})$. SVGP reduces the computational complexity and memory requirements of GPs from $\mathcal{O}(n^3)$ and $\mathcal{O}(n^2)$ to $\mathcal{O}(nm^2)$ and $\mathcal{O}(nm)$, respectively.
\subsection{Student-t Processes}
A Student-t Process (TP) \cite{shah2014student, solin2015state} is a variation of a Gaussian Process (GP) that utilizes the multivariate Student-t distribution as its base measure instead of the multivariate Gaussian distribution. TPs are advantageous for handling outliers and extreme values due to their heavier tail. This tail weight is controlled by the degrees of freedom parameter denoted as $\nu$. As $\nu$ tends to infinity, the Student-t distribution converges to a Gaussian distribution. By adjusting the $\nu$ value, the Student-t Process can represent a range of heavy-tailed processes, including the Gaussian Process. This flexibility allows the Student-t Process to effectively handle various data distributions with different tail weights.
\begin{definition}
The multivariate Student-t distribution with $\nu \in \mathbb{R}_{+}/\left[0,2\right]$ degrees of freedom, mean vector $\boldsymbol{\mu}$ and correlation matrix $\mathbf{R}$ is defined as:
\begin{equation}
\begin{aligned}
    &\mathcal{ST}(\mathbf{y} | \boldsymbol{\mu}, \mathbf{R}, \nu) \\=& \frac{\Gamma\left(\frac{\nu+n}{2}\right) |\mathbf{R}|^{-\frac{1}{2}}}{\Gamma\left(\frac{\nu}{2}\right) \left((\nu-2) \pi\right)^\frac{n}{2}} \left(1 + \frac{(\mathbf{y} - \boldsymbol{\mu})^\top \mathbf{R}^{-1} (\mathbf{y} - \boldsymbol{\mu})}{\nu-2}\right)^{-\frac{\nu+n}{2}},
\end{aligned}    
\end{equation}   
where $\Gamma(\cdot)$ is the gamma function, $\mathbf{y} \in \mathbb{R}^n$, and $|\cdot|$ denotes the determinant of a matrix.
\end{definition}
As we know, the multivariate Student-t distribution is consistent under marginalization. Using Kolmogorov's consistency theorem, we can define the Student-t process (TP) as follows.
\begin{definition}
A function $f$ is a Student-t process with parameters $\nu > 2$, mean function $\Psi : X \rightarrow \mathbb{R}$, and kernel function $k : X \times X \rightarrow \mathbb{R}$, if any finite collection of function values has a joint multivariate Student-t distribution, i.e. $(f(x_1),...,f(x_n))^T \sim \mathcal{T}\mathcal{P}_n(\nu, \phi, K)$ where $K \in \Pi_n$, $K_{ij} = k(x_i,x_j)$, $\phi_i = \Psi(x_i)$, and $\mathcal{T}\mathcal{P}_n(\nu, \phi, K)$ represents the multivariate Student-t distribution
\end{definition}
\subsection{Can Sparse Representation of GP be Extended to TP?}
We can naturally consider applying the sparse representation method to Student-t processes (TPs) since they also involve the inversion of an $n \times n$ matrix, which can be computationally expensive for large datasets. However, the extension of sparse representation to TPs is less explored compared to Gaussian processes (GPs), and there is no established theoretical framework or exploration of pseudo-inputs or inducing points in the existing literature on TPs. Nonetheless, given the similar computational challenges shared with GPs, it would be interesting to investigate and develop sparse representation techniques for TPs as well.

Fortunately, based on  previous research \cite{kotz2004multivariate}, we have found that the conditional distribution for a multivariate Student-t has an analytical form, similar to Gaussian distributions. This forms the cornerstone for defining inducing-inputs-based TP methods .
\begin{lemma}
\label{lemma1}
Suppose that $y \sim \mathcal{S}\mathcal{T}_n(\nu,\phi,K)$ is partitioned into $y_1$ and $y_2$ with dimensions $n_1$ and $n_2$, respectively. Let $\phi_1$, $\phi_2$, $K_{11}$, $K_{12}$, and $K_{22}$ denote the corresponding partitioned matrices, i.e, 
$$\phi = (\phi_1, \phi_2), \qquad K = \begin{pmatrix}K_{11} & K_{12} \\ K_{12}^T & K_{22}\end{pmatrix}.$$ 
Then, given $y_1$, the conditional distribution of $y_2$ can be expressed as $\mathcal{S}\mathcal{T}_{n_2}(\nu+n_1,\phi_2^*,\frac{\nu+\beta_1-2}{\nu+n_1-2}t^*)$, where $\phi_2^* = K_{21}K_{11}^{-1}(y_1 - \phi_1) + \phi_2, \beta_1 = (y_1 - \phi_1)^TK_{11}^{-1}(y_1 - \phi_1)$ and $t^* = K_{22} - K_{21}K_{11}^{-1}K_{12}$. The mean and covariance matrix of the conditional distribution of $y_2$ given $y_1$ are $\mathbb{E}[y_2|y_1] = \phi_2^*$ and $\text{cov}[y_2|y_1] = \frac{\nu+\beta_1-2}{\nu+n_1-2}t^*$, respectively.
\end{lemma}
Therefore, similar to GPs, we can define inducing inputs for TP priors based on this conditional distribution, to reduce the $\mathcal{O}(n^3)$ time complexity issue of TPs. Our proposed approach for sparse representation of TPs combines the sparsity-inducing properties of TP priors with variational inference, Monte Carlo sampling, and other methods, providing a more flexible and robust framework for modeling complex and heterogeneous data. The details of our method will be discussed in the following section.
\section{Sparse Student-t Processes and Variational Inference}

\subsection{Defining Inducing Points}
The key to the Sparse Student-t Process is the introduction of sparse inducing points, which summarize the information in the data and allow for more efficient computation. Specifically, we define a set of $M$ inducing points $Z$ whose function values $u$ follow a prior of a multivariate Student-t process with $\nu$ degrees of freedom and zero-mean,
\begin{equation}
    \mathbf{u}\sim \mathcal{T}\mathcal{P}_{M}(\nu,0,K_{Z,Z'})
\end{equation}

Formally, we assume a set of $N$ training data points $(X, \mathbf{y})$, where $\mathbf{y}=\left\{y_i\right\}_{i=1}^N$ is the output vector and $X=\left\{\mathbf{x}_i\right\}_{i=1}^N$ is the input matrix. In previous work, the noise model is incorporated into the covariance function by adding a noise covariance function to the parameterized kernel. In our model, on the other hand, we can explicitly define a noise term $\epsilon_i$, i.e. $y_i = f(\mathbf{x}_i)+\epsilon_i$. Let $\mathbf{f}=\left\{f_i\right\}_{i=1}^N$ and $\mathbf{u}=\left\{u_i\right\}_{i=1}^M$. Similar to the sparse Gaussian process , we define the kernel matrix $K$ as the variance matrix of the Student-t process. The computational complexity of any TP method scales with $\mathcal{O}(N^3)$ because of the need to invert the covariance matrix $K$. To reduce the computational complexity, we define the joint distribution of $\mathbf{f}$ and $\mathbf{u}$ as a multivariate Student-t distribution.

\begin{equation}
    p(\mathbf{u},\mathbf{f}) = \mathcal{ST}\left(\nu, \begin{bmatrix} \mathbf{0} \\ \mathbf{0} \end{bmatrix}, \begin{bmatrix} K_{Z,Z'} & K_{Z,X} \\ K_{X,Z} & K_{{X,X'}} \end{bmatrix}\right)
\end{equation}

By Lemma \ref{lemma1}, we can obtain the conditional probability distribution of $\mathbf{f}\mid \mathbf{u}$ as follows,
\begin{equation}
\label{condi}
p(\mathbf{f}|\mathbf{u})=\mathcal{ST}(\nu+M, \mu, \frac{\nu+\beta-2}{\nu+M-2}\Sigma)
\end{equation}
where $\mu=K_{X,Z}K_{Z,Z'}^{-1}\mathbf{u}, \beta=\mathbf{u}^T K_{Z,Z'}^{-1}\mathbf{u}$ and $\Sigma=K_{X,X'}-K_{X,Z}K_{Z,Z'}^{-1}K_{Z,X}$. This conditional distribution enables us to introduce the inducing points in a way that allows for efficient computation of the posterior distribution. Similarly to most Bayesian inference problems, TP also requires addressing the challenging partition function issue. Therefore, we contemplate utilizing variational inference for efficient optimization and resolution.
\subsection{Constructing Variational Lower Bound}
In variational inference, we often define a reasonable lower bound to the true marginal likelihood $\log p(\mathbf{y})$ due to its intractability. To efficiently optimize our model, we introduce variational distributions, denoted as $q(\mathbf{u})$, which approximate the true posteriors $p(\mathbf{u}|\mathbf{f},\mathbf{y})$. By utilizing Jensen's inequality, we can derive a lower bound that can be expressed as:
\begin{equation}
\label{elbo}
\begin{aligned}
    \log p(\mathbf{y}) &= \log \int p(\mathbf{y},\mathbf{f},\mathbf{u})d\mathbf{f}d\mathbf{u} \geq \mathcal{L}(q)\\&=\int p(\mathbf{f}|\mathbf{u})q(\mathbf{u}) \log\left(\frac{p(\mathbf{y}|\mathbf{f})p(\mathbf{f}|\mathbf{u})p(\mathbf{u})}{p(\mathbf{f}|\mathbf{u})q(\mathbf{u})}\right)d\mathbf{f} d\mathbf{u} \\&= \mathbb{E}_{p(\mathbf{f}|\mathbf{u})q(\mathbf{u})}[\log p(\mathbf{y}|\mathbf{f})] - \mathrm{KL}(q(\mathbf{u})||p(\mathbf{u}))
\end{aligned}
\end{equation}
 Once we have an approximate  posterior $q(\mathbf{u})$, we can employ Monte Carlo sampling and stochastic gradient-based optimization methods to estimate the variational lower bound $\mathcal{L}(q)$ and optimize the hyperparameters and inducing points. Next, we will discuss the construction of $q(\mathbf{u})$ and how to perform sampling from the posterior distribution.
\subsection{Computation and Optimization}

\subsubsection{Specifying the Posterior Distribution Family}
 As Eq. (\ref{elbo}) doesn't restrict the distribution family of $q(\mathbf{u})$, theoretically any analytically parameterized distribution family can be used to parameterize the variational distribution. We propose  to approximate $q(\mathbf{u})$ as a Student-t distribution, i.e. $q(\mathbf{u})=\mathcal{ST}(
\widetilde{\nu }, \mathbf{m},\mathbf{S})$  . Although the true posterior distribution is intractable, we intuitively consider that the posterior is closer to a Student's t-distribution because both the prior $p(\mathbf{u}) $ and conditional distribution $p(\mathbf{f} | \mathbf{u})$  are Student's t-distributions.

We rephrase the Eq. (\ref{elbo}) and summarize the two terms of our variational lower bound $\mathcal{L}(q)$. The statistical interpretations of these terms are the expected likelihood function and KL regularization. 
\begin{equation}
\label{elbo1}
    \mathcal{L} (q)=\underset{\mathrm{expected} \,\mathrm{likelihood}\, \mathrm{function}}{\underbrace{\mathbb{E} _{p(\mathbf{f}|\mathbf{u})q(\mathbf{u})}[\log p(\mathbf{y}|\mathbf{f})]}}-\underset{\mathrm{KL}\, \mathrm{regularization}}{\underbrace{\mathrm{KL(}q(\mathbf{u})||p(\mathbf{u}))}}
\end{equation}

\subsubsection{Re-parameterization Techniques}
In order to facilitate backpropagation, we  developed a reparameterization trick to sample from a multivariate Student-t distribution with arbitrary parameters. Firstly, we sample from a standard Gaussian distribution and a Gamma 
distribution  \cite{shah2014student,popescu2022matrix}, respectively, denoted as $\boldsymbol{\epsilon}_0 \sim \mathcal{N}(0,I)$ and $r^{-1}  \sim \Gamma(\nu / 2, 1 / 2)$. Then, we obtain a sample $\mathbf{u} \sim q(\mathbf{u})$ by applying an affine transformation $\mathbf{u}=(r(\nu-2)\mathbf{S})^{\frac{1}{2}}\boldsymbol{\epsilon}_0+\mathbf{m}$, where the $\frac{1}{2}$ power represents the Cholesky decomposition of a matrix. This allows us to compute the Monte Carlo estimate of  the expected likelihood function loss. Lemma \ref{lemmas} and Lemma \ref{lemma2} theoretically guarantees the feasibility of this method.
\begin{lemma}
\label{lemmas}
Let $K$ be an $n \times n$, symmetric, positive definite matrix, $\phi \in \mathbb{R}^n$, $\nu >0$, $\rho >0$, if
\begin{equation}
\begin{aligned}
r^{-1} & \sim \Gamma(\nu / 2, \rho / 2) \\
\boldsymbol{y} \mid r & \sim \mathrm{N}_{n}(\boldsymbol{\phi}, r(\nu-2) K / \rho),
\end{aligned}
\end{equation}
then, marginally $
\boldsymbol{y}\sim \mathcal{S} \mathcal{T} (\nu ,\phi ,K)$.   \end{lemma}
\begin{lemma}
\label{lemma2}
Consider 
$A \in \mathbb{R}^{n\times n}$, $b \in \mathbb{R}^n$, $\mu \in \mathbb{R}^n$, $\Sigma \in \mathbb{R}^{n\times n}$. Let $X$ be an $n$-dimensional Gaussian random vector s.t. $p(X)=\mathcal{N}(\mu,\Sigma)$. Then, the random
vector $Y = AX + b$ has the pdf:
$p(Y)=\mathcal{N}(A\mu + b,A^T\Sigma A)$
\end{lemma}
\subsubsection{The Numerical Computation of the KL Regularization Term}
The KL regularization term can be seen as an expectation of the log density ratio. Since both $p(\mathbf{u})$ and $q(\mathbf{u})$ are explicit density functions, we can naturally obtain:
\begin{equation}
\label{KL1}
    \mathrm{KL(}q(\mathbf{u})\parallel p(\mathbf{u}))=\mathbb{E} _{q(\mathbf{u})}[\log q(\mathbf{u})-\log p(\mathbf{u})]
\end{equation}
To compute Eq. (\ref{KL1}), a natural approach is to use Monte Carlo sampling to obtain a consistent estimate, where we still employ the reparameterization trick to sample from the posterior distribution $q(\mathbf{u})$ and then calculate the average of $\log q(\mathbf{u}) - \log p(\mathbf{u})$. In this paper, we refer to this method as SVTP-MC. By combining it with Eq. (\ref{elbo}), we can easily compute its evidence lower bound (ELBO).

However, since $q(\mathbf{u})$ is a high-dimensional distribution, Monte Carlo sampling may introduce bias in estimating Eq. (\ref{KL1}), especially in cases with a small number of samples. Therefore, in the following, we propose another method for regression tasks with relatively smaller training sets. The overall idea is to compute an explicit upper bound for Eq. (\ref{KL1}) and use it as a regularizer in place of Eq. (\ref{KL1}) as the ELBO.

First, we rewrite Eq. (\ref{KL1}) as follows:

\begin{equation}
\label{KL}
\begin{aligned}
	&\mathbb{E} _{q(\mathbf{u})}[\log q(\mathbf{u})-\log p(\mathbf{u})]=\frac{1}{2}\log \frac{|K_{Z,Z'}|}{|\mathbf{S}|}\\
&+\frac{M}{2}\log \frac{\nu -2}{\tilde{\nu}-2}+\log \Gamma \left( \frac{\tilde{\nu}+M}{2} \right)-\log \Gamma \left( \frac{\tilde{\nu}}{2} \right)\\
& -\log \Gamma \left( \frac{\nu +M}{2} \right) +\log \Gamma \left( \frac{\nu}{2} \right)-\frac{\tilde{\nu}+M}{2}\cdot\\ 
&\underset{\mathcal{L} _1}{\underbrace{\mathbb{E} _{q(\mathbf{u})}\left\{ \log \left[ 1+\frac{1}{\tilde{\nu}-2}\left( \mathbf{u}-\mathbf{m} \right) ^{\mathrm{T}}\mathbf{S}^{-1}\left( \mathbf{u}-\mathbf{m} \right) \right] \right\} }}\\
	&+\frac{\nu +M}{2}\underset{\mathcal{L} _2}{\underbrace{\mathbb{E} _{q(\mathbf{u})}\left\{ \log \left[ 1+\frac{1}{\nu -2}\mathbf{u}^{\mathrm{T}}K_{Z,Z'}^{-1}\mathbf{u} \right] \right\} }}\\
\end{aligned}
\end{equation}

Eq. (\ref{KL}) requires the computation of two expectations, $\mathcal{L}_1$ and $\mathcal{L}_2$, which represent the entropy of $q(\mathbf{u})$ and the cross-entropy between $q(\mathbf{u})$ and $p(\mathbf{u})$, respectively . First, for $\mathcal{L}_1$, we can establish its relationship with  the standard Student-t distribution using the following lemma:
\begin{lemma}
\label{lemma3}
Let $ p_X(\boldsymbol{x})=|\Sigma|^
{-\frac{1}{2}}p_{X_0}\{ \Sigma^{-\frac{1}{2}}
(\boldsymbol{x}-\mu) \}$ be a location-scale probability density function,
where $\mu \in \mathbb{R}^n$ is the location vector and $\Sigma \in  \mathbb{R}^ {n\times n}$ is the dispersion matrix. Let $X_0 =
\Sigma^{-\frac{1}{2}}
(X-\mu )$ be a standardized version of $X$, with standardized probability density function
$p_{X_0} (\boldsymbol{x}_0)$ that does not depend on $(\mu,\Sigma )$. Then, we have
\begin{equation}
    \mathbb{E}_{p(X)}[f(\Sigma^{-\frac{1}{2}}
(X-\mu ))]=\mathbb{E}_{p(X_0)}[f(X_0)]
\end{equation}
where $f$  is any arbitrary continuous function.
\end{lemma}

Next, we can use the following lemma to demonstrate that the entropy of the standard student-t distribution can be expressed in terms of the gamma function \cite{zografos1999maximum, villa2018objective} with respect to $\nu$:
\begin{lemma}
\label{lemma4}
Let $X_0$ be an N-dimensional standard student-t
random vector, i.e. $p(X_0)=\mathcal{ST}(\nu,0,I)$, then we have,
\begin{equation}
  \mathbb{E} _{p\left( X_0 \right)}\left[ \log \left( 1+\frac{{X_0}^{\top}X_0}{\nu-2} \right) \right] =\Psi \left( \frac{\nu +N}{2} \right) -\Psi \left( \frac{\nu}{2} \right) 
\end{equation}
where $\Psi(\cdot) $  is the digamma function, i.e. $\Psi \left( x \right) =\frac{d\log \varGamma \left( x \right)}{dx}$ for $x \in \mathbb{R}$. 
\end{lemma}

By utilizing this Lemma \ref{lemma3} and \ref{lemma4}, we can obtain a numerical solution for $\mathcal{L}_1$, that is,
\begin{equation}
\label{l1}
\mathcal{L}_1=\Psi \left( \frac{\Tilde{\nu} +M}{2} \right) -\Psi \left( \frac{\Tilde{\nu}}{2} \right)
\end{equation}

For $\mathcal{L}_2$, since it involves a challenging high-dimensional integration, we resort to Jensen's inequality to derive an upper bound, i.e.,
\begin{equation}
\begin{aligned}
    \mathcal{L}_2
&=\mathbb{E} _{q(\mathbf{u})}\left\{ \log \left[ 1+\frac{1}{\nu -2}\mathbf{u}^{\mathrm{T}}K_{Z,Z'}^{-1}\mathbf{u} \right] \right\} \\&\leqslant \log  \mathbb{E} _{q(\mathbf{u})}\left[ 1+\frac{1}{\nu -2}\mathbf{u}^{\mathrm{T}}K_{Z,Z'}^{-1}\mathbf{u} \right] 
\end{aligned}
\end{equation}
After the calculations, we present a tractable expression in the form of the following theorem:
\begin{theorem}
 An upper bound for $\mathcal{L}_2$, denoted as $\mathcal{L}_2^\star$, can be expressed as follows: 
 \begin{small}
     \begin{equation}
 \label{l2}
     \mathcal{L}_2^\star=
\log \left\{ 1+\frac{1}{\nu-2}\mathrm{Tr}\left( K_{Z,Z'}^{-1}\mathbf{S} \right) +\frac{1}{\nu -2}\mathrm{Tr}\left( K_{Z,Z'}^{-1}\mathbf{mm}^T \right) \right\} 
\end{equation}
\end{small}
\end{theorem}

In this paper, we refer to this approach as SVTP-UB because it computes an upper bound for the KL regularization term. Regarding the applicability of these two methods, the paper suggests that since the main purpose of the KL regularization term is to prevent overfitting, which often occurs when there is insufficient data, it is recommended to use SVTP-UB with a larger regularization term in regression tasks with smaller datasets. On the other hand, for regression tasks with larger datasets, it is recommended to use SVTP-MC.
\subsubsection{Stochastic Gradient Descent}
Based on Eq. (\ref{l1}) and Eq. (\ref{l2}), we derive a new lower bound on $\log p(\mathbf{y})$ as follows:
\begin{equation}
\begin{aligned}
    \log p(\mathbf{y})\geqslant \mathcal{L} ^{\star}(q)&=\mathbb{E} _{p(\mathbf{f}|\mathbf{u})q(\mathbf{u})}[\log p(\mathbf{y}|\mathbf{f})]-C\left( \nu ,\tilde{\nu},\mathbf{S} \right) \\&+\frac{\tilde{\nu}+M}{2}\mathcal{L} _1-\frac{\nu +M}{2}\mathcal{L} _{2}^{\star}
\end{aligned}
\end{equation}

By exploiting the independence of $y_i$, we can express $\log p(\mathbf{y}|\mathbf{f})$ as the sum of individual terms:
\begin{equation}
\log p(\mathbf{y}|\mathbf{f})=\sum\nolimits_{i=1}^n{\log p(y_i|f_i)}
\end{equation}

According to Eq. (\ref{condi}), the marginal distribution of $f_i|\mathbf{u}$ can be defined as $p(f_i|\mathbf{u})=\mathcal{ST}(\nu+M,\mu_i,\frac{\nu+\beta-2}{\nu+M-2}\Sigma_i)$, where $\mu_i=K_{\mathbf{x_i},Z}K_{Z,Z'}^{-1}\mathbf{u}$, $\beta=\mathbf{u}^T K_{Z,Z'}^{-1}\mathbf{u}$, and $\Sigma=K_{\mathbf{x_i},\mathbf{x_i}}-K_{\mathbf{x_i},Z}K_{Z,Z'}^{-1}K_{Z,\mathbf{x_i}}$. It is worth noting that $f_i$ only depends on the corresponding inputs $\mathbf{x_i}$.

Therefore, we can approximate the entire dataset by leveraging the idea of stochastic gradient descent and using a mini-batch of data as follows:
\begin{equation}
\begin{aligned}
    &\sum\nolimits_{i=1}^n{\mathbb{E} _{p(f_i|\mathbf{u})q(\mathbf{u})}\log p(y_i|f_i)}\\\approx& \frac{n}{B}\sum\nolimits_{i=1}^B{\mathbb{E} _{p(f_i|\mathbf{u})q(\mathbf{u})}\log p(y_i|f_i)}
\end{aligned}
\end{equation}

\subsection{Prediction and Relationship between SVGP}

 We present algorithms for posterior sampling and prediction using our proposed Sparse Variational Student-t Process (SVTP). Specially, we first sample $\mathbf{u}$ from $q(\mathbf{u})=\mathcal{ST}(
\widetilde{\nu }, \mathbf{m},\mathbf{S})$, then sample $\mathbf{f}$ from $p(\mathbf{f}|\mathbf{u})$. To make predictions at new input points $X_*$, we compute the mean and covariance matrix of the predictive distribution:
\begin{equation}
    \begin{aligned}
    \mu_*&=\boldsymbol{k}_*^TK_{Z,Z'}^{-1}\mathbf{u}\\
    \Sigma_*&=\frac{\nu+\beta-2}{\nu+M-2}(k_{**} - \boldsymbol{k}_*^TK_{Z,Z'} ^{-1}\boldsymbol{k}_*)
    \end{aligned}
\end{equation}
where $\beta=\mathbf{u}^T K_{Z,Z'}^{-1}\mathbf{u}$, $\boldsymbol{k}_*$ denotes the kernel vector between inducing points $Z$ and the new input points $X_*$, $k_{**}$ is the kernel between the new input points. We can finally obtain the predicted distribution $\mathbf{f}^*\sim\mathcal{ST}(\nu+M,\mu_*,\Sigma_*)$, $\mathbf{y}^*=\mathbf{f}^*+ \boldsymbol{\epsilon}$, where $\boldsymbol{\epsilon}=\left\{\epsilon_i\right\}_{i=1}^N$  . 

Specifically, we have discovered an interesting fact that when we set $q(\mathbf{u})=\mathcal{ST}(
\nu +n, \mathbf{m},\mathbf{S})$ , we can analytically marginalize $\mathbf{u}$, resulting in:
\begin{equation}
    q\left( \mathbf{f} \right) =\int{p\left( \mathbf{f}|\mathbf{u} \right) q\left( \mathbf{u} \right)}d\mathbf{u}=\mathcal{ST}\left( \nu ,\mu ',\Sigma ' \right) 
\end{equation}
where $\mu'=K_{X,Z}K_{Z,Z'}^{-1}\mathbf{m}$ and $\Sigma'=K_{X,X'}-K_{X,Z}K_{Z,Z'}^{-1}(K_{Z,Z'}-\mathbf{S})K_{Z,Z'}^{-1}K_{Z,X}$.

In this case, the mean and covariance matrix of the marginal distribution $q( \mathbf{f})$ are consistent with SVGP. Furthermore, we have shown in the following theorem the relationship between SVGP and SVTP, stating that SVGP is an extreme case of SVTP.
\begin{theorem}
  as $\nu\rightarrow\infty$, the posterior distribution of SVTP converges in distribution to the posterior distribution of SVGP.
\end{theorem}
For a Student-t distribution, the parameter $\nu$ controls the degree of tail-heaviness. Smaller $\nu$ values correspond to heavier tails, while larger $\nu$ values make the tails more closely resemble those of a Gaussian distribution. Therefore, it is a natural extension to apply Student-t distributions to sparse scenarios.
\section{Outliers in Regression Analysis}
When discussing the differences between SVTP  and SVGP  in handling outliers in regression analysis, 
following \cite{hensman2015scalable}, the variational lower bound for SVGP can be expressed as:
\begin{equation}
\mathcal{L}_{\text{SVGP}}=
\mathbb{E} _{q\left( \mathbf{u} \right)}\left[ \log  p\left( \mathbf{y}|\mathbf{u} \right) \right] -\mathrm{KL}\left( q\left( \mathbf{u} \right) ||p\left( \mathbf{u} \right) \right) 
\end{equation}
For the first term,
\begin{equation}
\label{elbogp}
\begin{aligned}
    &\mathbb{E} _{q\left( \mathbf{u} \right)}\left[ \log  p\left( \mathbf{y}|\mathbf{u} \right) \right] \\=&-\frac{1}{2}\mathbb{E} _{q\left( \mathbf{u} \right)}\left[ \left( \mathbf{y}-\mu \right) ^TS^{-1}\left( \mathbf{y}-\mu \right) \right] -\frac{1}{2}\log |\Sigma |+C
\end{aligned}    
\end{equation}
where $\mu,\Sigma$ are defined as in Eq.(\ref{condi}), and C is a constant independent of $(X, \mathbf{y})$.
The first term of the ELBO for SVTP, obtained from the previous section, can be defined as:
\begin{equation}
\label{elbotp}
\begin{aligned}
    &\mathbb{E} _{q\left( \mathbf{u} \right)}\left[ \log  p\left( \mathbf{y}|\mathbf{u} \right) \right] \\=&-\frac{\nu +n}{2}\mathbb{E} _{q\left( \mathbf{u} \right)}\left[ \log  \left( 1+\frac{\left( \mathbf{y}-\mu \right) ^TS^{-1}\left( \mathbf{y}-\mu \right)}{\nu -2} \right) \right] \\&-\frac{1}{2}\log |\Sigma |+C
\end{aligned}    
\end{equation}
We can see that the only difference between Eq.(\ref{elbogp}) and Eq. (\ref{elbotp}) is
that the term $
\left( \mathbf{y}-\mu \right) ^TS^{-1}\left( \mathbf{y}-\mu \right) 
$ in Eq.(\ref{elbogp}) is the result
of a log transformation of the term  in Eq.(\ref{elbotp}). If there
are input outliers, this term  would be disturbed and
the log transformation can reduce the disturbance. Therefore,
the ELBO of SVTP  is
more robust to input outliers than that of SVGP.

\section{Experiments}

We evaluate the performance of our proposed methods, SVTP-UB and SVTP-MC, on a total of eight real-world datasets obtained from UCI and Kaggle. These datasets are used for comparison against two baseline models, namely Sparse Variational Gaussian Processes (SVGP) and the full Student-T Processes (TP). Our experiments are designed to verify the following four
propositions:
\begin{itemize}
    \item Compared to TP, our proposed sparse representation method SVTP significantly reduces computational complexity, which facilitates the extension of TP method to higher-dimensional datasets.

    \item As an inducing point method, SVTP demonstrates higher regression accuracy and more robust uncertainty estimation compared to SVGP.
    \item SVTP has a greater advantage over traditional SVGP when dealing with outlier data.
    \item SVTP-MC demonstrates superior performance on larger datasets, while SVTP-UB serves as an effective method to prevent overfitting on smaller datasets.
\end{itemize}

\subsection{Datasets and  Experimental Setup}
\label{Datasets}

We use  eight data sets to carry out the experiments, for which details can be seen in the appendix.
We utilize all algorithms with a maximum iteration number of 5000 to minimize the negative ELBO. We set the batch size to 1024. The learning rate is set to 0.01, and the data is standardized. The noise term $\epsilon_i$ 
  is fixed at 0.10 in the experiments for fair comparison of the performance of each model . We opt to use the PyTorch platform and conduct all experiments on a single NVIDIA  A100 GPU. 

\begin{table}[t]
\label{time1}

\centering

\resizebox{\linewidth}{!}{
\begin{tabular}{|c|c|c|c|c|}

\hline
Dataset & Datasize & $m=n/10$&$m=n/4$  & Full TP\\
\hline
Yacht & (307,6) & 0.042s& 0.045s& 0.058s \\
\hline
Energy &(767,8)  & 0.043s& 0.046s& 0.075s \\
\hline
Boston & (505,13) & 0.042s& 0.045s& 0.062s \\
\hline
Elevator & (16599,18) & 0.931s& 2.412s& 97.977s \\
\hline
Concrete & (1029,8) & 0.050s& 0.061s& 0.110s \\
\hline
Protein & (45730,9) & 12.44s& 97.93s& - \\
\hline
Kin8nm & (8192,8) &0.45s & 0.98s& 5.49s \\
\hline

Taxi & (209673,7) & -& -& - \\
\hline
\end{tabular}}
\caption{ Comparison of Computational Complexity on regression experiments. In this context, $m$ refers to the number of inducing points. "$-$" stands for not having enough space and computational power to calculate its time.}

\end{table}

\subsection{Comparison of Computational Complexity between SVTP and Full TP}

\begin{table*}[ht]
\label{mse}
\centering
\begin{tabular*}{\textwidth}{@{\extracolsep{\fill}}|c|*{6}{c|}}
\hline
\multirow{2}{*}{Dataset} & \multicolumn{3}{c|}{MSE} &  \multicolumn{3}{c|}{LL} \\ \cline{2-7}
 & SVGP & SVTP+UB & SVTP+MC & SVGP & SVTP+UB & SVTP+MC \\ \hline
Yacht & $7.56\pm0.23$ & $\mathbf{1.44}\pm0.02$ & $1.55\pm0.03$ & $-485\pm2.4$ & $\mathbf{-114}\pm0.5$ &$-122\pm0.7$\\ \hline
Energy &  $2.02 \pm 0.06$  &  $\mathbf{0.74} \pm 0.04$ & $0.84 \pm 0.05$ & $-156\pm 1.8$ & $\mathbf{-63.3}\pm 1.2$ & $-70.5\pm 1.3 $ \\ \hline
Boston & $3.17\pm0.21$ & $\mathbf{1.82}\pm0.10$ & $2.01 \pm 0.12$ & $-171\pm4.8$ & $-\mathbf{126}\pm2.6$ & $-140\pm2.9 $\\ \hline
Elevator & $6.54 \pm 0.28$ & $1.93\pm 0.15 $& $\mathbf{1.84}\pm 0.11$ & $-391 \pm 19.2$ & $-126.1\pm10.2$   & $-\mathbf{106} \pm 8.7$ \\ \hline
Concrete & $7.48\pm0.24$ & $5.71\pm0.20$ & $\mathbf{3.23}\pm0.18$ & $-478 \pm 11.4$ & $-435\pm 9.0$ & $\mathbf{-276}\pm7.4$ \\ \hline
Protein & $7.23\pm0.26$ & $5.64\pm0.22$ & $\mathbf{4.86}\pm0.14$ & $-425\pm19.3$ & $-295\pm15.5$ & $\mathbf{-260}\pm9.8$ \\ \hline
Kin8nm & $5.19\pm0.52 $& $4.35\pm0.47 $&$ \mathbf{3.85}\pm0.46$ & $-400\pm28.6$ & $-385\pm22.5$ & $-\mathbf{225}\pm19.4$ \\ \hline
Taxi & $0.76\pm0.07$ & $0.54\pm0.05$ & $\mathbf{0.41}\pm0.04$ & $-301\pm8.9$ & $-152\pm6.4$ & $-\mathbf{124}\pm4.6$ \\ \hline
Concrete$\_$Outliers & $12.20\pm 0.32$ & $8.62\pm0.25$ & $\mathbf{6.40}\pm 0.22$& $-658\pm 20.8$ & $-471.0\pm14.2$ & $-\mathbf{336.3}\pm13.6$\\ \hline
Kin8nm$\_$Outliers & $14.66\pm2.25$ & $9.12\pm 1.72$ & $\mathbf{8.86}\pm 0.65$ & $-1169\pm61.9$ & $-622\pm52.4$ & $-\mathbf{501}\pm31.4 $ \\ \hline
\end{tabular*}
\caption{Predictive Mean Squared Errors (MSE) and test Log Likelihoods (LL) of regression experiments}
\end{table*}
To empirically validate the computational efficiency of SVTP, we conducted experiments on eight different  datasets. We compared the average time taken by SVTP-UB for completing one epoch (i.e., iterating over the entire dataset) with that of the conventional full TP method. The results are presented in Table 1, showing a significant reduction in computational complexity with the decreasing number of inducing points in SVTP. Consequently, the average time taken for one epoch decreases accordingly, thus substantiating the method's effectiveness in terms of computational efficiency.

\subsection{Real-world Dataset Regression}

 Unlike the baseline method proposed by  \cite{shah2014student}, our approach significantly reduces computational complexity, allowing us to test our method on the entire dataset instead of just a small subset as a training set. 

Each experiment was repeated using five-fold cross-validation, and the average value and range were reported for each repetition. The evaluation metrics used to assess the performance are the Mean Squared Error (MSE), and test Log Likelihood (LL).  We selected the number of inducing points to be one-fourth of the original data size and employed the squared exponential kernel. For large datasets such as Protein and Taxi, the number of inducing points are selected to be one-fourth of the batchsize. The experimental findings, presented in Table 2, demonstrate that our SVTP methods outperform SVGP across all datasets. This empirically establishes the improved accuracy and uncertainty estimation of SVTP.
\subsection{Regression Experimental Analysis}
As discussed in the previous sections, in order to investigate the superior performance of the SVTP, we conducted empirical density analysis and kernel density estimation on  the standardized y for all datasets. The results, depicted in Figure 1, revealed the presence of varying degrees of outliers, irregularities, and heavy-tailed characteristics in many of the datasets. Among them, Yacht and Taxi exhibited particularly prominent features. It is noteworthy that SVTP outperformed SVGP to a significant extent on these two datasets, further validating the effectiveness of their approach.

\begin{figure}[t]
\label{figure1}
  \centering
  \includegraphics[width=\linewidth]{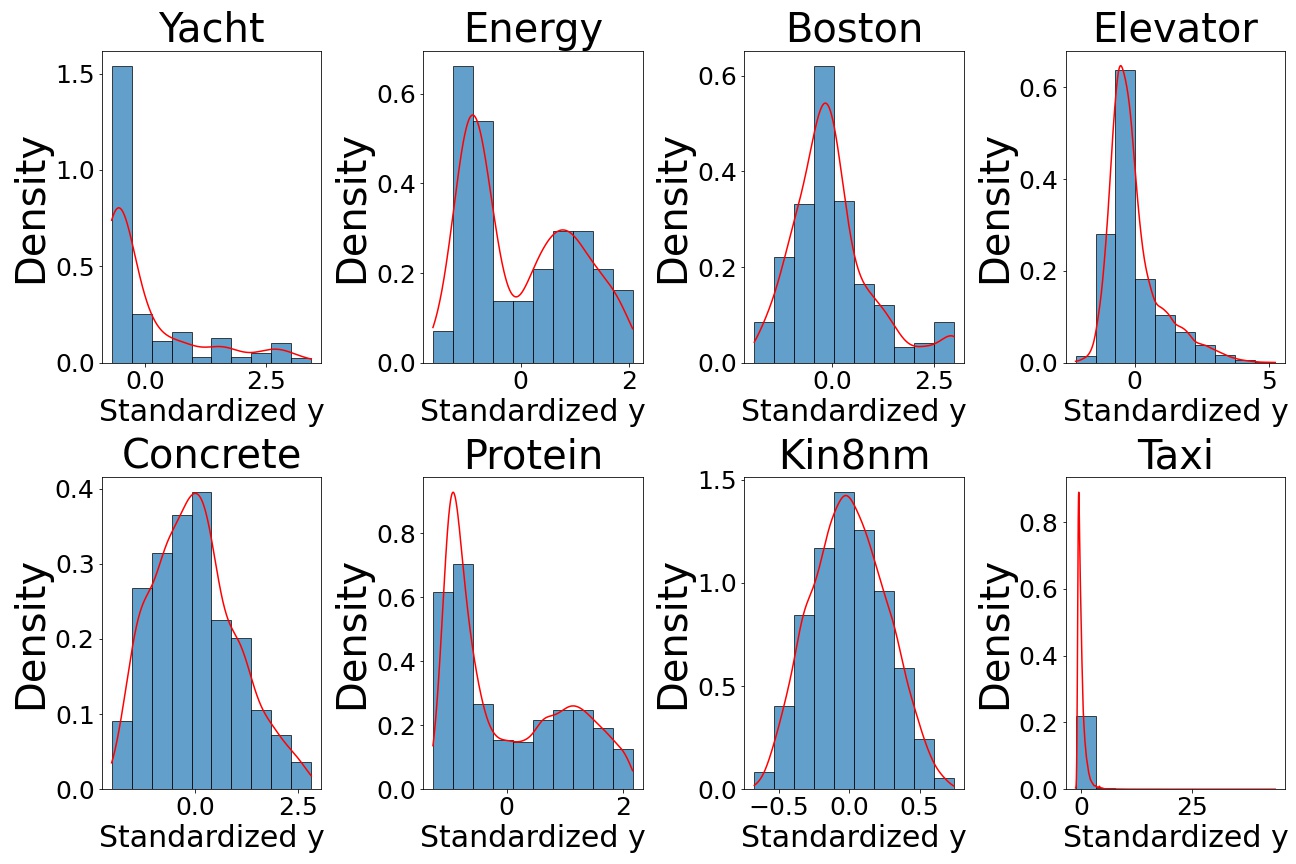}
  \caption{Empirical density analysis and kernel density estimation on the experimental datasets}
  \label{fig:image}
\end{figure}
Additionally, we selected two datasets, Concrete and Kind8nm, to perform outlier regression. We randomly added three times the standard deviation to the y-values of 5$\%$ of the data to synthesize data with more outliers.The results of this analysis are also presented in Table 1, demonstrating the performance of the model on these datasets.

In this section, we also conducted an empirical analysis of the differences between SVTP-UB and SVTP-MC methods in computing the KL regularization term. We selected four datasets: Boston, Elevator, Concrete, and Protein, and obtained the specific values of the KL regularization term at convergence using these two algorithms. We report the results in Table 3. Upon analysis, we observed that SVTP-UB tends to produce larger estimates of the KL regularization term. Considering the findings from Table 1, where SVTP-UB generally outperforms SVTP-MC in terms of convergence for large datasets, we can deduce that SVTP-UB's tighter constraints negatively impact convergence in this scenario. On the other hand, for smaller datasets, where overfitting is a common concern due to limited data availability, these tighter constraints provided by SVTP-UB become crucial in preventing overfitting. Therefore, taking into account the impact of dataset size on overfitting and convergence, we recommend employing SVTP-MC for larger datasets. This method strikes a better balance between the regularization term and the dataset size, enabling smoother convergence. Conversely, for smaller datasets where the risk of overfitting is more pronounced, SVTP-UB should be the preferred option.

\begin{table}[t]
\centering

\renewcommand{\arraystretch}{0.25} 
\resizebox{\linewidth}{!}{
\begin{tabular}{|c|c|c|}
\hline
{\tiny Datasets} &  {\tiny SVGP+UB} & {\tiny SVGP+MC}\\

\hline
{\tiny Boston} &  {\tiny 56} &{\tiny 2.2} \\
\hline
{\tiny Elevator} &   {\tiny 42}& {\tiny 3.1}\\
\hline
{\tiny Concrete} &   {\tiny 78}& {\tiny 1.3} \\
\hline
{\tiny Protein} &   {\tiny 39}& {\tiny 6.3} \\
\hline

\end{tabular}
}
\caption{Comparison of the average values of the KL-term on all datasets}
\end{table}

\section{Conclusion}

In this paper, we proposed SVTP, a novel approach that extends the benefits of sparse representation from Gaussian Processes to Student-t Processes. The effectiveness of the approach is demonstrated through experiments on both synthetic and real-world datasets. The results show that the Sparse Student-t Process can handle large-scale datasets while also accurately modeling non-Gaussian and heavy-tailed distributions. In terms of future work, we suggest exploring the application of the Sparse Student-t Process to other likelihoods, such as classification and multi-output regression. While the current focus of the paper is on modeling regression tasks, extending the approach to handle different types of data and problems could further enhance its versatility and applicability.
\section{Acknowledgments}
The work is supported by the Fundamental Research Program of Guangdong, China, under Grant 2023A1515011281; and in part by the National Natural Science Foundation of China under Grant 61571005.
\appendix
\section{Proof of Lemmas and Theorems.}
\begin{lemma}
\label{lemma1}
Suppose that $y \sim \mathcal{S}\mathcal{T}_n(\nu,\phi,K)$ is partitioned into $y_1$ and $y_2$ with dimensions $n_1$ and $n_2$, respectively. Let $\phi_1$, $\phi_2$, $K_{11}$, $K_{12}$, and $K_{22}$ denote the corresponding partitioned matrices, i.e, 
$$\phi = (\phi_1, \phi_2), \qquad K = \begin{pmatrix}K_{11} & K_{12} \\ K_{12}^T & K_{22}\end{pmatrix}.$$ 
Then, given $y_1$, the conditional distribution of $y_2$ can be expressed as $\mathcal{S}\mathcal{T}_{n_2}(\nu+n_1,\phi_2^*,\frac{\nu+\beta_1-2}{\nu+n_1-2}t^*)$, where $\phi_2^* = K_{21}K_{11}^{-1}(y_1 - \phi_1) + \phi_2, \beta_1 = (y_1 - \phi_1)^TK_{11}^{-1}(y_1 - \phi_1)$ and $t^* = K_{22} - K_{21}K_{11}^{-1}K_{12}$. The mean and covariance matrix of the conditional distribution of $y_2$ given $y_1$ are $\mathbb{E}[y_2|y_1] = \phi_2^*$ and $\text{cov}[y_2|y_1] = \frac{\nu+\beta_1-2}{\nu+n_1-2}t^*$, respectively.
\end{lemma}
\begin{proof}
 Let $\beta_2=\left(\boldsymbol{y}_{\mathbf{2}}-\boldsymbol{\phi}^*_2\right)^{\top} t^{*-1}\left(\boldsymbol{y}_2-\boldsymbol{\phi}^*_2\right)$. Note that $\beta_1+\beta_2=(\boldsymbol{y}-\boldsymbol{\phi})^{\top} K^{-1}(\boldsymbol{y}-\boldsymbol{\phi})$. We have
\begin{equation*}
\begin{aligned}
&p\left(\boldsymbol{y}_{\mathbf{2}} \mid \boldsymbol{y}_{\mathbf{1}}\right)=\frac{p\left(\boldsymbol{y}_{\mathbf{1}}, \boldsymbol{y}_{\mathbf{2}}\right)}{p\left(\boldsymbol{y}_{\mathbf{1}}\right)} \\& \propto\left(1+\frac{\beta_1+\beta_2}{\nu-2}\right)^{-(\nu+n) / 2}\left(1+\frac{\beta_1}{\nu-2}\right)^{\left(\nu+n_1\right) / 2} \\
& \propto\left(1+\frac{\beta_2}{\beta_1+\nu-2}\right)^{-(\nu+n) / 2}
\end{aligned}
\end{equation*}
Comparing this expression to the definition of a $\mathcal{ST}$ density function gives the required result.
\end{proof}

\begin{lemma}
\label{lemmas}
Let $K$ be an $n \times n$, symmetric, positive definite matrix, $\phi \in \mathbb{R}^n$, $\nu >0$, $\rho >0$, if
\begin{equation}
\begin{aligned}
r^{-1} & \sim \Gamma(\nu / 2, \rho / 2) \\
\boldsymbol{y} \mid r & \sim \mathrm{N}_{n}(\boldsymbol{\phi}, r(\nu-2) K / \rho),
\end{aligned}
\end{equation}
then, marginally $
\boldsymbol{y}\sim \mathcal{S} \mathcal{T} (\nu ,\phi ,K)$.   \end{lemma}
\begin{proof}
     Let $\beta=(\boldsymbol{y}-\boldsymbol{\phi})^{\top} K^{-1}(\boldsymbol{y}-\boldsymbol{\phi})$. We can analytically marginalize out the scalar $r$,
$$
\begin{aligned}
&p(\boldsymbol{y})=\int p(\boldsymbol{y} \mid r) p(r) d r\\ & \propto \int \exp \left(-\frac{\rho \beta}{2(\nu-2) r}\right) r^{-\frac{n}{2}} \exp \left(-\frac{\rho}{2 r}\right) r^{-\frac{(\nu+2)}{2}} d r \\
& \propto\left(1+\frac{\beta}{\nu-2}\right)^{-\frac{(\nu+n)}{2}} \int \exp \left(-\frac{1}{2 r}\right) r^{-\frac{(\nu+n+2)}{2}} d r \\
& \propto\left(1+\frac{\beta}{\nu-2}\right)^{-\frac{(\nu+n)}{2}}
\end{aligned}
$$
Hence $\boldsymbol{y} \sim \mathcal{ST}(\nu, \boldsymbol{\phi}, K)$.
\end{proof}

\begin{lemma}

Consider 
$A \in \mathbb{R}^{n\times n}$, $b \in \mathbb{R}^n$, $\mu \in \mathbb{R}^n$, $\Sigma \in \mathbb{R}^{n\times n}$. Let $X$ be an $n$-dimensional Gaussian random vector s.t. $p(X)=\mathcal{N}(\mu,\Sigma)$. Then, the random
vector $Y = AX + b$ has the pdf:
$p(Y)=\mathcal{N}(A\mu + b,A^T\Sigma A)$
\end{lemma}

\begin{proof}
 Let $\phi_X(t)=\mathbb{E}\left[\exp \left(i t^{\top} X\right)\right]$ be the characteristic function of a random variable $X \in \mathbb{R}^n$.
If $X$ is normally distributed $X \sim \mathcal{N}(\mu, \Sigma)$, then we have $\phi_X(t)=\exp \left(i t^{\top} \mu-\frac{1}{2} t^{\top} \Sigma t\right)$
If $Y = AX + b$, then
$$
\begin{aligned}
\phi_Y(t) & =\mathbb{E}\left[\exp \left(i t^{\top}(A X+b)\right)\right] \\
& =\mathbb{E}\left[\exp \left(i t^{\top} b\right) \exp \left(i t^{\top} A X\right)\right] \\
& =\exp \left(i t^{\top} b\right) \mathbb{E}\left[\exp \left(i\left(A^{\top} t\right)^{\top} X\right)\right] \\
& =\exp \left(i t^{\top} b\right) \phi_X\left(A^{\top} t\right) \\
& =\exp \left(i t^{\top} b\right) \exp \left(i\left(A^{\top} t\right)^{\top} \mu-\frac{1}{2}\left(A^{\top} t\right)^{\top} \Sigma\left(A^{\top} t\right)\right) \\
& =\exp \left(i t^{\top}(A \mu+b)-\frac{1}{2} t^{\top} A \Sigma A^{\top} t\right)
\end{aligned}
$$
Since the characteristic function uniquely defines the distribution, we have $Y \sim \mathcal{N}\left(A \mu+b, A \Sigma A^{\top}\right)$ .   
\end{proof}

\begin{lemma}
\label{lemma3}
     Let $ p_X(\boldsymbol{x})=|\Sigma|^
{-\frac{1}{2}}p_{X_0}\{ \Sigma^{-\frac{1}{2}}
(\boldsymbol{x}-\mu) \}$ be a location-scale probability density function,
where $\mu \in \mathbb{R}^n$ is the location vector and $\Sigma \in  \mathbb{R}^ {n\times n}$ is the dispersion matrix. Let $X_0 =
\Sigma^{-\frac{1}{2}}
(X-\mu )$ be a standardized version of $X$, with standardized probability density function
$p_{X_0} (\boldsymbol{x}_0)$ that does not depend on $(\mu,\Sigma )$. Then, we have
\begin{equation}
    \mathbb{E}_{p(X)}[f(\Sigma^{-\frac{1}{2}}
(X-\mu ))]=\mathbb{E}_{p(X_0)}[f(X_0)]
\end{equation}
where $f$  is any arbitrary continuous function.
\end{lemma}
\begin{proof}
\begin{equation*}
\begin{aligned}
\mathbb{E} _{p(X)}[f(\Sigma ^{-\frac{1}{2}}(X-\mu ))]&=\int{p(X)}\left[ f(X_0) \right] dX\\&=\int{p(X_0)}\left[ f(X_0) \right] dX_0
\\
&=\mathbb{E} _{p(X_0)}[f(X_0)]
\end{aligned}
\end{equation*}
    
\end{proof}

\begin{lemma}

let $X_0$ be an N-dimensional standard student-t
random vector, i.e. $p(X_0)=\mathcal{ST}(\nu,0,I)$, then we have,
\begin{equation}
  \mathbb{E} _{p\left( X_0 \right)}\left[ \log \left( 1+\frac{{X_0}^{\top}X_0}{\nu-2} \right) \right] =\Psi \left( \frac{\nu +N}{2} \right) -\Psi \left( \frac{\nu}{2} \right) 
\end{equation}
where $\Psi(\cdot) $  is the digamma function, i.e. $\Psi \left( x \right) =\frac{d\log \varGamma \left( x \right)}{dx}$ for $x \in \mathbb{R}$. 
\end{lemma}
\begin{proof}
    This is a direct conclusion from \cite{kotz2004multivariate,villa2018objective}.
\end{proof}

\begin{theorem}
 An upper bound for $\mathcal{L}_2$, denoted as $\mathcal{L}_2^\star$, can be expressed as follows: 
 \begin{small}
     \begin{equation*}
 \label{l2}
     \mathcal{L}_2^\star=
\log \left\{ 1+\frac{1}{\nu-2}\mathrm{Tr}\left( K_{Z,Z'}^{-1}\mathbf{S} \right) +\frac{1}{\nu -2}\mathrm{Tr}\left( K_{Z,Z'}^{-1}\mathbf{mm}^T \right) \right\} 
\end{equation*}
\end{small}
\end{theorem}
\begin{proof}
\begin{tiny}
    
\begin{equation*}
\begin{aligned}
	&\mathcal{L}_2=\mathbb{E} _{q(\mathbf{u})}\left\{ \log \left[ 1+\frac{1}{\nu -2}\mathbf{u}^{\mathrm{T}}K_{Z,Z'}^{-1}\mathbf{u} \right] \right\}\\
	\le &\log \left\{ \mathbb{E} _{q(\mathbf{u})}\left[ 1+\frac{1}{\nu -2}\mathbf{u}^{\mathrm{T}}K_{Z,Z'}^{-1}\mathbf{u} \right] \right\}\\
	=&\log \left\{ \mathbb{E} _{q(\mathbf{u})}\left[ \begin{array}{c}
	1+\frac{1}{\nu -2}\left( \mathbf{u}-\mathbf{m}+\mathbf{m} \right) ^{\mathrm{T}}K_{Z,Z'}^{-1}\left( \mathbf{u}-\mathbf{m}+\mathbf{m} \right)\\
\end{array} \right] \right\}\\
	=&\log \left\{ \mathbb{E} _{q(\mathbf{u})}\left[ 1+\frac{1}{\nu -2}\mathbf{u}^{\mathrm{T}}K_{Z,Z'}^{-1}\mathbf{u}+\frac{1}{\nu -2}\mathbf{m}^{\mathrm{T}}K_{Z,Z'}^{-1}\mathbf{m} \right. \right.\\
	&\left. \left. +\frac{1}{\nu -2}\left( \mathbf{u}-\mathbf{m} \right) ^{\mathrm{T}}K_{Z,Z'}^{-1}\mathbf{m}+\frac{1}{\nu -2}\mathbf{m}^{\mathrm{T}}K_{Z,Z'}^{-1}\left( \mathbf{u}-\mathbf{m} \right) \right] \right\}\\
	=&\log \left\{ \mathbb{E} _{q(\mathbf{u})}\left\{ 1+\frac{1}{\nu -2}\mathrm{tr}\left[ K_{Z,Z'}^{-1}\left( \mathbf{u}-\mathbf{m} \right) \left( \mathbf{u}-\mathbf{m} \right) ^{\mathrm{T}} \right] +\frac{1}{\nu -2}\mathrm{tr}\left[ K_{Z,Z'}^{-1}\mathbf{mm}^{\mathrm{T}} \right] \right\} \right\}\\
	=&\log \left\{ 1+\frac{1}{\nu -2}\mathrm{tr}\left( K_{Z,Z'}^{-1}\mathbf{S} \right) +\frac{1}{\nu -2}\mathrm{tr}\left[ K_{Z,Z'}^{-1}\mathbf{mm}^{\mathrm{T}} \right] \right\}=\mathcal{L}_2^*\\
\end{aligned}
\end{equation*}
\end{tiny}
\end{proof}
\begin{lemma}
     Suppose $ f \sim \mathcal{T} \mathcal{P}(\nu, \phi, K)$  and $g \sim \mathcal{G} \mathcal{P}(\phi, K) \text {. Then } f \text { tends to } g \text { in distribution as } \nu \rightarrow \infty$ . 
\end{lemma}

\begin{proof}
     It is sufficient to show convergence in density for any finite collection of inputs. Let $\boldsymbol{y} \sim \mathcal{ST}_n(\nu, \boldsymbol{\phi}, K)$ and set $\beta=(\boldsymbol{y}-\boldsymbol{\phi})^{\top} K^{-1}(\boldsymbol{y}-\boldsymbol{\phi})$ then
$$
p(\boldsymbol{y}) \propto\left(1+\frac{\beta}{\nu-2}\right)^{-(\nu+n) / 2} \rightarrow e^{-\beta / 2}
$$
as $\nu \rightarrow \infty$. Hence the distribution of $\boldsymbol{y}$ tends to a $\mathrm{N}_n(\phi, K)$ distribution as $\nu \rightarrow \infty$.
\end{proof}
\begin{theorem}
  as $\nu\rightarrow\infty$, the posterior distribution of SVTP converges in distribution to the posterior distribution of SVGP.
\end{theorem}
\begin{proof}
 In SVTP, the joint probability density of $\mathbf{u}$ and $\mathbf{f}$ can be written as,
 $$
p(\mathbf{u},\mathbf{f})=\mathcal{S} \mathcal{T} \left( \nu ,\left[ \begin{array}{c}
	0\\
	0\\
\end{array} \right] ,\left[ \begin{matrix}
	K_{Z,Z'}&		K_{Z,X}\\
	K_{X,Z}&		K_{X,X'}\\
\end{matrix} \right] \right) 
$$
as $\nu\rightarrow\infty$,$$
\mathbf{f},\mathbf{u}\xrightarrow{\mathcal{P}}\mathbf{f}',\mathbf{u}'\sim \mathcal{N} \left( \left[ \begin{array}{c}
	0\\
	0\\
\end{array} \right] ,\left[ \begin{matrix}
	K_{Z,Z'}&		K_{Z,X}\\
	K_{X,Z}&		K_{X,X'}\\
\end{matrix} \right] \right) 
$$

 By Bayes Rule, the the posterior distribution of SVTP can be written as,
 \begin{equation*}
     p\left( \mathbf{u},\mathbf{f}|\mathbf{y} \right) =\frac{p\left( \mathbf{f},\mathbf{u} \right) p\left( \mathbf{y}|\mathbf{f} \right)}{p\left( \mathbf{y} \right)}
\end{equation*}
As $\nu \rightarrow \infty$, due to the continuity of random variables,  we have $$
\mathbf{u},\mathbf{f}|\mathbf{y}  \xrightarrow{\mathcal{P}} \mathbf{u}',\mathbf{f}'|\mathbf{y}  
$$

\end{proof}
\begin{proposition}
Specifically, we have discovered an interesting fact that when we set $q(\mathbf{u})=\mathcal{ST}(
\nu +n, \mathbf{m},\mathbf{S})$ , we can analytically marginalize $\mathbf{u}$, resulting in :
\begin{equation}
    q\left( \mathbf{f} \right) =\int{p\left( \mathbf{f}|\mathbf{u} \right) q\left( \mathbf{u} \right)}d\mathbf{u}=\mathcal{ST}\left( \nu ,\mu ',\Sigma ' \right) 
\end{equation}
where $\mu'=K_{X,Z}K_{Z,Z'}^{-1}\mathbf{m}$ and $\Sigma'=K_{X,X'}-K_{X,Z}K_{Z,Z'}^{-1}(K_{Z,Z'}-\mathbf{S})K_{Z,Z'}^{-1}K_{Z,X}$.
\end{proposition}
\begin{proof}
Let $
\beta _1=\left( \mathbf{u}-\mathbf{m} \right) ^{\top}\mathbf{S}^{-1}\left( \mathbf{u}-\mathbf{m} \right) 
$, and
$\beta _2=\left( \mathbf{f}-\mu  \right) ^{\top}\Sigma ^{-1}\left( \mathbf{f}-\mu  \right) $, where $\mu=K_{X,Z}K_{Z,Z'}^{-1}\mathbf{u}$ and $\Sigma=K_{X,X'}-K_{X,Z}K_{Z,Z'}^{-1}K_{Z,X}$. Then we have,
\begin{equation*}
\begin{aligned}
    & p\left( \mathbf{f}|\mathbf{u} \right) q\left( \mathbf{u} \right) \\&\propto \left( 1+\frac{\beta _2}{\beta _1+\nu -2} \right) ^{-(\nu +n+M)/2}\left( 1+\frac{\beta _1}{\nu -2} \right) ^{-\left( \nu +n+M \right) /2}
    \\& \propto \left( 1+\frac{\beta _1+\beta _2}{\nu -2} \right) ^{-\left( \nu +n+M \right) /2}
\end{aligned} 
\end{equation*}
The above form is a multivariate Student's t-distribution. By the properties of the Mahalanobis distance, we have,
\begin{equation*}
\begin{aligned}
    \beta _1+\beta _2=\left[ \begin{matrix}
	\mathbf{u}-\mathbf{m}&		\mathbf{f}-\mu '\\
\end{matrix} \right] \left[ \begin{matrix}
	\mathbf{S}&		\Sigma _{12}\\
	\Sigma _{21}&		\Sigma '\\
\end{matrix} \right] ^{-1}\left[ \begin{array}{c}
	\mathbf{u}-\mathbf{m}\\
	\mathbf{f}-\mu '\\
\end{array} \right] 
\\
\end{aligned}
\end{equation*}
where $\mu'=K_{X,Z}K_{Z,Z'}^{-1}\mathbf{m}$ and $\Sigma'=K_{X,X'}-K_{X,Z}K_{Z,Z'}^{-1}(K_{Z,Z'}-\mathbf{S})K_{Z,Z'}^{-1}K_{Z,X}$, $\Sigma_{21}=\Sigma_{12}^T=K_{X,Z}K_{Z,Z'}^{-1}\mathbf{S}$ is the covariance of $\mathbf{u}$ and $\mathbf{f}$. From this, we obtain the marginal distribution.
\begin{equation*}
    q\left( \mathbf{f} \right)=\mathcal{ST}\left( \nu ,\mu ',\Sigma ' \right) 
\end{equation*}
\end{proof}

\section{Experiments Datasets}
We obtained eight datasets from UCi and Kaggle to conduct the experiments, here is the information of the datasets.
\begin{itemize}
\item \textbf{Concrete Data}. This data set is about the slump flow of
concrete. It contains 1030 instances and 9 attributes.
\item \textbf{Boston Data}. The Boston housing data was collected in 1978 and each of the 506 entries represent aggregated data
about 14 features for homes from various suburbs in Boston,
Massachusetts.

\item \textbf{Kin8m Data} This is a data set concerned with the forward
kinematics of an 8 link robot arm. Among the existing variants
of this data set it has used the variant 8nm, which is known
to be highly non-linear and medium noisy.

\item \textbf{Yacht Data}Yacht dataset is used to predict the hydodynamic
 performance of sailing yachts from dimensions and
velocity. It comprises 308 instances and 6 features performed
at the Delft Ship Hydromechanics Laboratory. 

\item \textbf{Energy Data} It performs energy analysis using 12 different building shapes simulated in Ecotect. The buildings differ with respect to the glazing area, the glazing area distribution, and
the orientation, amongst other parameters. It simulates various settings as functions of the afore-mentioned characteristics to obtain 768 building shapes. The dataset comprises 768 samples
and 8 features, aiming to predict two real valued responses.
\item \textbf{Elevator Data} The problem has 18 attributes and this data
set is obtained from the task of controlling a F16 aircraft,
although the target variable and attributes are different from
the ailerons domain. In this case the goal variable is related
to an action taken on the elevators of the aircraft.
\item \textbf{Protein Data} This is a data set of Physicochemical Properties
 of Protein Tertiary Structure. The data set is taken from
CASP 5-9. There are 45730 decoys and size varying from 0
to 21 armstrong.
\item \textbf{Taxi Trip Fare Data} This project involves analyzing millions of New York City yellow taxi cab trips using a BigQuery Public Dataset. The objective is to build a machine learning model within BigQuery that can predict the fare of a cab ride based on the model inputs.
\end{itemize}

\bibliography{aaai24}

\begin{thebibliography}{27}
\providecommand{\natexlab}[1]{#1}

\bibitem[{Andrade(2023)}]{andrade2023robustness}
Andrade, J. A.~A. 2023.
\newblock On the robustness to outliers of the Student-t process.
\newblock \emph{Scandinavian Journal of Statistics}, 50(2): 725--749.

\bibitem[{Bankestad et~al.(2020)Bankestad, Sj{\"o}lund, Taghia, and Sch{\"o}n}]{bankestad2020elliptical}
Bankestad, M.; Sj{\"o}lund, J.; Taghia, J.; and Sch{\"o}n, T. 2020.
\newblock The Elliptical Processes: a New Family of Flexible Stochastic Processes.
\newblock \emph{arXiv preprint arXiv:2003.07201}.

\bibitem[{Blei, Kucukelbir, and McAuliffe(2017)}]{blei2017variational}
Blei, D.~M.; Kucukelbir, A.; and McAuliffe, J.~D. 2017.
\newblock Variational inference: A review for statisticians.
\newblock \emph{Journal of the American statistical Association}, 112(518): 859--877.

\bibitem[{Blomqvist, Kaski, and Heinonen(2020)}]{blomqvist2020deep}
Blomqvist, K.; Kaski, S.; and Heinonen, M. 2020.
\newblock Deep convolutional Gaussian processes.
\newblock In \emph{Machine Learning and Knowledge Discovery in Databases: European Conference, ECML PKDD 2019, W{\"u}rzburg, Germany, September 16--20, 2019, Proceedings, Part II}, 582--597. Springer.

\bibitem[{Chen, Wang, and Gorban(2020)}]{chen2020multivariate}
Chen, Z.; Wang, B.; and Gorban, A.~N. 2020.
\newblock Multivariate Gaussian and Student-t process regression for multi-output prediction.
\newblock \emph{Neural Computing and Applications}, 32: 3005--3028.

\bibitem[{Cutajar et~al.(2017)Cutajar, Bonilla, Michiardi, and Filippone}]{cutajar2017random}
Cutajar, K.; Bonilla, E.~V.; Michiardi, P.; and Filippone, M. 2017.
\newblock Random feature expansions for deep Gaussian processes.
\newblock In \emph{International Conference on Machine Learning}, 884--893. PMLR.

\bibitem[{Deisenroth, Fox, and Rasmussen(2013)}]{deisenroth2013gaussian}
Deisenroth, M.~P.; Fox, D.; and Rasmussen, C.~E. 2013.
\newblock Gaussian processes for data-efficient learning in robotics and control.
\newblock \emph{IEEE transactions on pattern analysis and machine intelligence}, 37(2): 408--423.

\bibitem[{Heinonen et~al.(2018)Heinonen, Yildiz, Mannerstr{\"o}m, Intosalmi, and L{\"a}hdesm{\"a}ki}]{heinonen2018learning}
Heinonen, M.; Yildiz, C.; Mannerstr{\"o}m, H.; Intosalmi, J.; and L{\"a}hdesm{\"a}ki, H. 2018.
\newblock Learning unknown ODE models with Gaussian processes.
\newblock In \emph{International conference on machine learning}, 1959--1968. PMLR.

\bibitem[{Hensman, Matthews, and Ghahramani(2015)}]{hensman2015scalable}
Hensman, J.; Matthews, A.; and Ghahramani, Z. 2015.
\newblock Scalable variational Gaussian process classification.
\newblock In \emph{Artificial Intelligence and Statistics}, 351--360. PMLR.

\bibitem[{Ketkar and Ketkar(2017)}]{ketkar2017stochastic}
Ketkar, N.; and Ketkar, N. 2017.
\newblock Stochastic gradient descent.
\newblock \emph{Deep learning with Python: A hands-on introduction}, 113--132.

\bibitem[{Kingma and Welling(2013)}]{kingma2013auto}
Kingma, D.~P.; and Welling, M. 2013.
\newblock Auto-encoding variational bayes.
\newblock \emph{arXiv preprint arXiv:1312.6114}.

\bibitem[{Kotz and Nadarajah(2004)}]{kotz2004multivariate}
Kotz, S.; and Nadarajah, S. 2004.
\newblock \emph{Multivariate t-distributions and their applications}.
\newblock Cambridge University Press.

\bibitem[{L{\'a}zaro-Gredilla and Figueiras-Vidal(2009)}]{lazaro2009inter}
L{\'a}zaro-Gredilla, M.; and Figueiras-Vidal, A. 2009.
\newblock Inter-domain Gaussian processes for sparse inference using inducing features.
\newblock \emph{Advances in Neural Information Processing Systems}, 22.

\bibitem[{L{\'a}zaro-Gredilla et~al.(2010)L{\'a}zaro-Gredilla, Quinonero-Candela, Rasmussen, and Figueiras-Vidal}]{lazaro2010sparse}
L{\'a}zaro-Gredilla, M.; Quinonero-Candela, J.; Rasmussen, C.~E.; and Figueiras-Vidal, A.~R. 2010.
\newblock Sparse spectrum Gaussian process regression.
\newblock \emph{The Journal of Machine Learning Research}, 11: 1865--1881.

\bibitem[{Lee et~al.(2022)Lee, Feng, Humt, M{\"u}ller, and Triebel}]{lee2022trust}
Lee, J.; Feng, J.; Humt, M.; M{\"u}ller, M.~G.; and Triebel, R. 2022.
\newblock Trust your robots! predictive uncertainty estimation of neural networks with sparse gaussian processes.
\newblock In \emph{Conference on Robot Learning}, 1168--1179. PMLR.

\bibitem[{Moss, Ober, and Picheny(2023)}]{moss2023inducing}
Moss, H.~B.; Ober, S.~W.; and Picheny, V. 2023.
\newblock Inducing point allocation for sparse gaussian processes in high-throughput bayesian optimisation.
\newblock In \emph{International Conference on Artificial Intelligence and Statistics}, 5213--5230. PMLR.

\bibitem[{Popescu, Glocker, and van~der Wilk(2022)}]{popescu2022matrix}
Popescu, S.; Glocker, B.; and van~der Wilk, M. 2022.
\newblock Matrix Inversion free variational inference in Conditional Student's T Processes.
\newblock In \emph{Fourth Symposium on Advances in Approximate Bayesian Inference}.

\bibitem[{Quinonero-Candela and Rasmussen(2005)}]{quinonero2005unifying}
Quinonero-Candela, J.; and Rasmussen, C.~E. 2005.
\newblock A unifying view of sparse approximate Gaussian process regression.
\newblock \emph{The Journal of Machine Learning Research}, 6: 1939--1959.

\bibitem[{Rasmussen(2003)}]{rasmussen2003gaussian}
Rasmussen, C.~E. 2003.
\newblock Gaussian processes in machine learning.
\newblock In \emph{Summer school on machine learning}, 63--71. Springer.

\bibitem[{Roth(2012)}]{roth2012multivariate}
Roth, M. 2012.
\newblock \emph{On the multivariate t distribution}.
\newblock Link{\"o}ping University Electronic Press.

\bibitem[{Salimbeni and Deisenroth(2017)}]{salimbeni2017doubly}
Salimbeni, H.; and Deisenroth, M. 2017.
\newblock Doubly stochastic variational inference for deep Gaussian processes.
\newblock \emph{Advances in neural information processing systems}, 30.

\bibitem[{Shah, Wilson, and Ghahramani(2014)}]{shah2014student}
Shah, A.; Wilson, A.; and Ghahramani, Z. 2014.
\newblock Student-t processes as alternatives to Gaussian processes.
\newblock In \emph{Artificial intelligence and statistics}, 877--885. PMLR.

\bibitem[{Solin and S{\"a}rkk{\"a}(2015)}]{solin2015state}
Solin, A.; and S{\"a}rkk{\"a}, S. 2015.
\newblock State space methods for efficient inference in Student-t process regression.
\newblock In \emph{Artificial Intelligence and Statistics}, 885--893. PMLR.

\bibitem[{Tang et~al.(2017)Tang, Niu, Wang, Dai, An, Cai, and Xia}]{tang2017student}
Tang, Q.; Niu, L.; Wang, Y.; Dai, T.; An, W.; Cai, J.; and Xia, S.-T. 2017.
\newblock Student-t Process Regression with Student-t Likelihood.
\newblock In \emph{IJCAI}, 2822--2828.

\bibitem[{Titsias(2009)}]{titsias2009variational}
Titsias, M. 2009.
\newblock Variational learning of inducing variables in sparse Gaussian processes.
\newblock In \emph{Artificial intelligence and statistics}, 567--574. PMLR.

\bibitem[{Villa and Rubio(2018)}]{villa2018objective}
Villa, C.; and Rubio, F.~J. 2018.
\newblock Objective priors for the number of degrees of freedom of a multivariate t distribution and the t-copula.
\newblock \emph{Computational Statistics \& Data Analysis}, 124: 197--219.

\bibitem[{Zografos et~al.(1999)}]{zografos1999maximum}
Zografos, K.; et~al. 1999.
\newblock On maximum entropy characterization of Pearson's type II and VII multivariate distributions.
\newblock \emph{Journal of Multivariate Analysis}, 71(1): 67--75.

\end{thebibliography}
\end{document}